\documentclass[letterpaper, 10 pt, conference]{ieeeconf}  

\IEEEoverridecommandlockouts                              

\usepackage{graphics} 
\usepackage{epsfig} 
\usepackage{amsmath} 
\usepackage{amssymb}  
\usepackage{caption}
\usepackage{subcaption}
\usepackage{graphicx}
\usepackage{bm}
\usepackage{algorithm}
\usepackage{algorithmic}
\usepackage{multirow}
\usepackage{array}
\usepackage{booktabs}
\usepackage{newtxtext,newtxmath}
\newtheorem{definition}{\textbf{Definition}}
\newtheorem{proposition}{\textbf{Proposition}}
\newtheorem{lemma}{\textbf{Lemma}}

\title{\LARGE \bf
Symmetry-Guided Multi-Agent Inverse Reinforcement Learning
}

\author{Yongkai Tian$^{1}$,
Yirong Qi$^{1}$,
Xin Yu$^{1}$,
Wenjun Wu$^{1}$,
Jie Luo$^{1,*}$
\thanks{*Corresponding Author: Jie Luo {\tt\small luojie@buaa.edu.cn.}}
\thanks{$^{1}$State Key Laboratory of Complex \& Critical Software Environment, Beihang University, Beijing, China}%
}

\begin{document}

\maketitle
\thispagestyle{empty}
\pagestyle{empty}

\begin{abstract}
In robotic systems, the performance of reinforcement learning depends on the rationality of predefined reward functions. However, manually designed reward functions often lead to policy failures due to inaccuracies. Inverse Reinforcement Learning (IRL) addresses this problem by inferring implicit reward functions from expert demonstrations. Nevertheless, existing methods rely heavily on large amounts of expert demonstrations to accurately recover the reward function. The high cost of collecting expert demonstrations in robotic applications, particularly in multi-robot systems, severely hinders the practical deployment of IRL. Consequently, improving sample efficiency has emerged as a critical challenge in multi-agent inverse reinforcement learning (MIRL). Inspired by the symmetry inherent in multi-agent systems, this work theoretically demonstrates that leveraging symmetry enables the recovery of more accurate reward functions. Building upon this insight, we propose a universal framework that integrates symmetry into existing multi-agent adversarial IRL algorithms, thereby significantly enhancing sample efficiency. Experimental results from multiple challenging tasks have demonstrated the effectiveness of this framework. Further validation in physical multi-robot systems has shown the practicality of our method. 
\end{abstract}


\section{Introduction}
In recent years, Multi-Agent Reinforcement Learning (MARL) has made notable achievements in multi-robot systems \cite{feng2023mact, feng2024hierarchical, gu2023safe, liao2025sigma}. The efficacy of MARL primarily arises from its ability to optimize manually designed reward functions. However, this also highlights a critical challenge: devising suitable reward functions for complex and poorly defined tasks is often impractical \cite{hadfield2017inverse, amodei2016concrete}. Furthermore, improperly designed reward functions may result in undesirable agent behavior. To mitigate these problems, Multi-Agent Imitation Learning has emerged as a prevalent algorithm \cite{natarajan2010multi, lin2017multiagent, ravichandar2020recent}, capable of directly learning the policy from expert demonstrations, thus avoiding the need for a detailed reward function design.


\begin{figure}[thpb]
  \centering
  \includegraphics[width=\columnwidth]{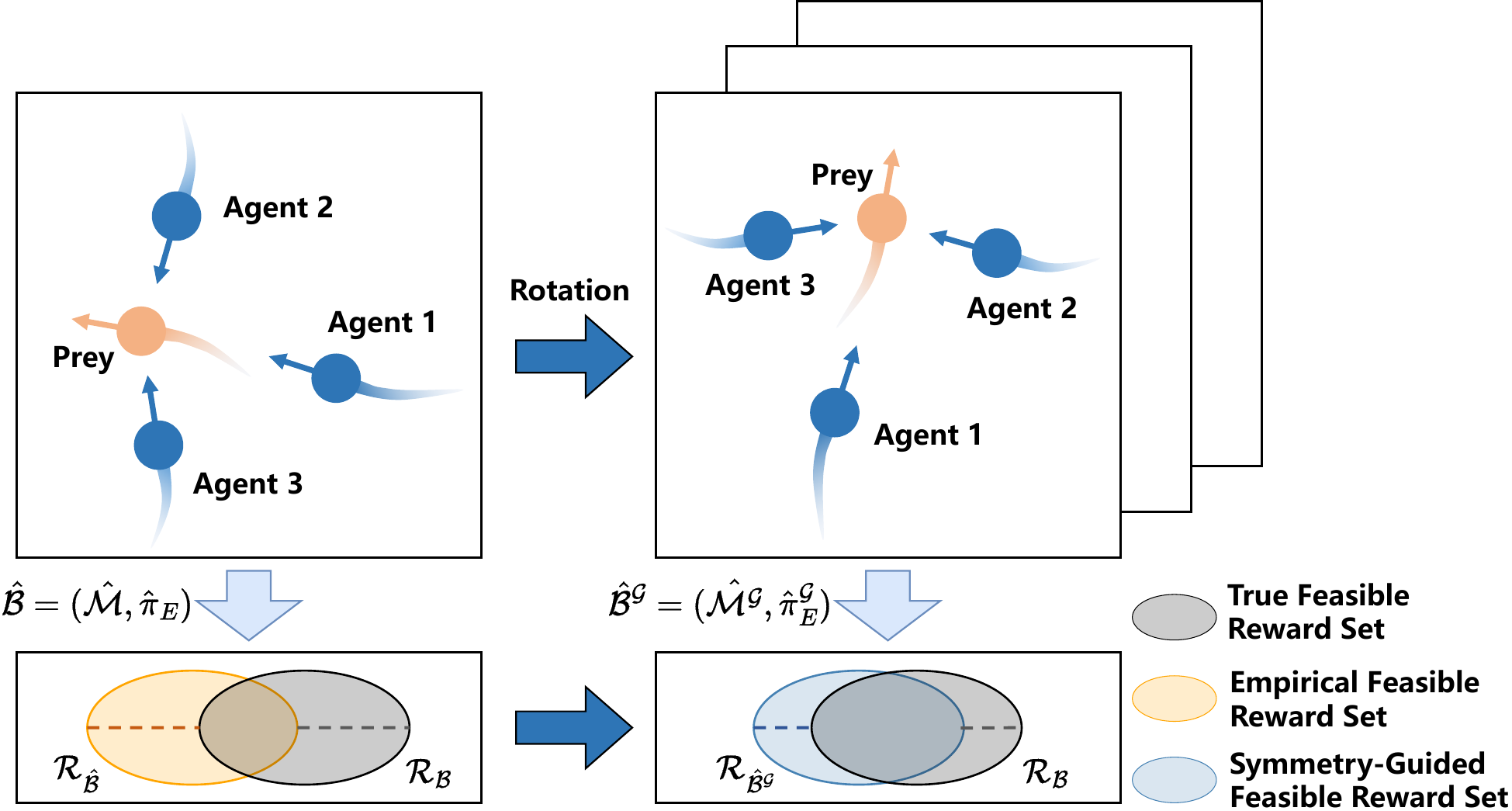}
  \caption{Multi-agent inverse reinforcement learning with symmetry. When the global state undergoes rotation, the agents' actions correspondingly rotate. This property can be utilized to recover a more accurate feasible reward set.}
  \label{symmetry}
\end{figure}

In the field of imitation learning, two predominant methodologies are employed: Behavior Cloning (BC) \cite{pomerleau1991efficient} and Inverse Reinforcement Learning (IRL) \cite{ng2000algorithms}. BC leverages supervised learning to derive policies from expert state-action pairs. Although BC has achieved success in some simple scenarios, it suffers from compounded errors caused by the covariate shift \cite{ross2010efficient, ross2011reduction}. On the other hand, IRL recovers the implicit reward function from expert demonstrations and then uses this reward function to train the policy, allowing it to be applied to more complex tasks. However, one significant challenge it faces is the need for large-scale and diverse expert demonstrations, leading to sample inefficiency problems \cite{zare2024survey}. In multi-robot systems, due to the increased complexity of the system, the higher cost of collecting expert demonstrations and the need for more expert demonstrations to train the model make the problem of sample inefficiency even more pronounced.

Addressing the challenge of sample inefficiency in Multi-Agent IRL (MIRL) fundamentally involves reducing the reliance on expert demonstrations and recovering more accurate reward functions from fewer samples. Utilizing prior knowledge is an effective way that allows neural networks to learn more general knowledge from limited datasets, thereby significantly improving sample efficiency \cite{karniadakis2021physics}. Multi-agent systems (MAS) commonly exhibit properties known as symmetry, including equivariance and invariance. For example, in Pursuit task depicted in Fig. \ref{symmetry}, when the environment state rotates, the agent's actions rotate accordingly, demonstrating equivariance. In contrast, metrics such as the distance between agents remain unchanged, exhibiting invariance. Symmetry not only serves as a fundamental property of MAS but also plays a crucial role in multi-agent learning. However, its application in MIRL has not been widely explored. While symmetry has been demonstrated to improve algorithmic performance in the context of MARL \cite{van2021multi,yu2023esp, tian2024exploiting}, its theoretical basis does not apply to MIRL. 

In this paper, we present the first proof that leveraging the symmetry of MAS can recover more accurate reward functions, bridging the theoretical gap in the performance improvement of MIRL using symmetry. Based on this theory, we propose a novel Symmetry-Guided Framework (SGF) for multi-agent adversarial IRL algorithms, aimed at improving the sample efficiency. This framework comprises two primary modules: a Symmetry-Guided Demonstration Augmenter (SGDA) and a Symmetry-Aware Discriminator (SAD). The SGDA performs geometric transformations on expert and generated demonstrations to expand the dataset. The SAD is used to enhance the ability to recognize symmetric data. The key contributions of this paper can be summarized as follows:
\begin{itemize}
    \item We provide theoretical proof that leveraging symmetry can recover better reward functions.
    \item We propose a Symmetry-Guided Framework for multi-agent adversarial IRL, including the Symmetry-Guided Demonstration Augmenter and Symmetry-Aware Discriminator.
    \item We validate our method across multiple multi-agent tasks, demonstrating its ability to learn effective policies from fewer expert demonstrations.
    \item We deploy our method in a physical multi-robot environment, confirming its effectiveness in the real world.
\end{itemize}

\section{Related Work}
\subsection{Multi-Agent Adversarial Inverse Reinforcement Learning}
Multi-agent generative adversarial imitation learning (MA-GAIL) \cite{song2018multi} employs a mini-max objective to train agents by alternating between optimizing discriminators $D_1, \ldots, D_N$ and policies $\bm{\pi}$. These discriminators discern between agent-generated and expert-generated state-action pairs, with agent policies refined via MARL using rewards derived from the discriminators. Multi-agent Adversarial Inverse Reinforcement Learning (MA-AIRL) \cite{yu2019multi} modifies the discriminators inspired by the logistic stochastic best response equilibrium (LSBRE). Additionally, MA-AIRL defines reward functions as $\log D_i(s, a_i, s^{\prime}) - \log(1 - D_i(s, a_i, s^{\prime}))$, effectively capturing the reward dynamics leading to expert behaviors. Numerous subsequent works are proposed based on these two algorithms. \cite{jeon2020scalable} applies an attention mechanism to the generator to enhance the scalability of MA-GAIL. \cite{yu2021swarm} extends AIRL to a multi-agent framework using parameter sharing and explores its application in the biological domain. \cite{sengadu2023dec} adapts AIRL to the distributed domain.

\subsection{Symmetry in Multi-Agent Learning}
Recent studies have validated the effectiveness of applying symmetry in MARL. \cite{yu2023esp} utilized rotational symmetry for data augmentation and symmetric consistency loss to improve sample efficiency in MARL. \cite{yu2024leveraging} extended \cite{yu2023esp} to include scenarios of imperfect symmetry. Moreover, \cite{10611035} introduced an adaptive data augmentation method to boost sample efficiency. \cite{jianye2022boosting} and \cite{li2021permutation} developed policy networks that ensure permutation invariance, while \cite{van2021multi} designed policy networks with rotational symmetry. \cite{tian2024exploiting} leveraged the hierarchical symmetry of MAS to improve algorithm performance. Although the effectiveness of symmetry has been validated in MARL, its theory does not apply to MIRL. Only \cite{xu2023face} incorporates symmetry into the GAN framework similar to ours; however, it is not a MIRL algorithm. The application of symmetry in MIRL has not been deeply explored. This paper is the first to leverage symmetry to improve the sample efficiency of MIRL algorithms and provides theoretical proof for this method.


\section{Preliminaries}

\subsection{Equivariance and Invariance}
In MAS, symmetry often appears as equivariance and invariance to transformations. Let $L_g: X \to X$ represent a transformation corresponding to an element $g$ of some abstract group $\mathcal{G}$. A function $f: X \to Y$ exhibits equivariance to the transformation $g$ if there exists a related transformation $K_g: Y \to Y$ in the output space that satisfies:
\[
f(L_g(\mathbf{x})) = K_g(f(\mathbf{x})).
\]
Furthermore, the function $f$ is considered invariant for $g$ if it meets the following condition:
\[
f(L_g(\mathbf{x})) = f(\mathbf{x}),
\]
indicating that the output of $f$ does not change despite applying $L_g$.


\subsection{Markov Game}
In this paper, we model the MAS as a Markov Game \cite{littman1994markov} without reward (MG\textbackslash R), which is defined by the tuple $ \mathcal{M} = (\mathcal{N}, \mathcal{S}, \{\mathcal{A}_i\}_{i=1}^N, P, \gamma, \bm{\pi}) $. $\mathcal{N}$ is the set of agents and $ \mathcal{S} $ denotes the state space. Each agent $ i $ has an associated action space $ \mathcal{A}_i $, with the joint action space being the Cartesian product $ \mathcal{A} = \mathcal{A}_1 \times \mathcal{A}_2 \times \cdots \times \mathcal{A}_N $. The state transition function $ P: \mathcal{S} \times \mathcal{A} \times \mathcal{S} \rightarrow [0, 1] $ defines the probability distribution over the next state given the current state and joint action. $\gamma$ is the discount factor. $\bm{\pi}(\bm{a}|s)=\prod_{i=1}^N \pi_i(a_i|s)$ represents the joint policy, where we use bold variables without subscript $i$ to denote the concatenation of all agents ($e.g.$ $\bm{\pi}$ denotes the joint policy and $\bm{a}$ denotes the actions of all agents). Given an MG\textbackslash R $ \mathcal{M} $ and a reward function $r \in \mathbb{R}^{\mathcal{S} \times \mathcal{A}} $, we use $ \mathcal{M} \cup r $ to denote the MG obtained by pairing $ \mathcal{M} $ and $r$.

\section{Symmetry-Guided Multi-Agent Inverse Reinforcement Learning}
\label{sec4}
\subsection{Problem Definition}
To avoid reliance on a specific MIRL algorithm, we discuss the set of all feasible reward functions for the expert's policy. The MIRL problem is defined by a tuple $\mathcal{B} = (\mathcal{M}, \bm{\pi}_E)$, where $\mathcal{M}$ is an MG\textbackslash R and $\bm{\pi}_E$ is an expert joint policy. Essentially, the MIRL problem involves recovering a reward function $r \in \mathbb{R}^{\mathcal{S} \times \mathcal{A}}$ such that the expert joint policy $\bm{\pi}_E$ is optimal for $\mathcal{M} \cup r$. Typically, the reward function that satisfies this condition is not unique. The collection of all such reward functions that meet this criterion is referred to as the feasible reward set, denoted as $\mathcal{R}_\mathcal{B}$ for the MIRL problem $\mathcal{B}$. In \cite{metelli2021provably}, the feasible reward set for the single-agent IRL problem is characterized. Here, we extend analogous findings to the multi-agent context. 
\begin{lemma}[Feasible Reward Set]
\label{lemma1}
\textit{A reward function $r$ is feasible for $\mathcal{B}$ if and only if there exists $\zeta \in \mathbb{R}^{\mathcal{S} \times \mathcal{A}}_{\geq 0}$ and $V\in \mathbb{R}^{\mathcal{S}}$ such that for all $s$, $\bm{a}$ it holds:
\begin{equation*}
    r(s, \bm{a})=-\zeta\mathbb{I}_{\{\bm{\pi}_E(\bm{a}|s)=0\}}+V(s)-\gamma\sum\nolimits_{s^\prime} P(s^\prime|s, \bm{a})V(s^\prime),
\end{equation*}
where $\mathbb{I}$ denotes the indicator function, which takes the value 1 when $\bm{\pi}_E(\bm{a}|s) = 0$, and 0 otherwise.}
\end{lemma}

Here, the first term depends only on the expert joint policy $\bm{\pi}_E$ and does not depend on the MG. The other parameter can be interpreted as reward-shaping through the value function $V$.

\subsection{Error Propagation}
Next, we study the error propagation. In practice, IRL is trained using a transition tuples dataset $\tau_E = \{(s_j, \bm{a}_j, s^\prime_j)\}_{j=1}^M$ collected by the expert policy, where $s^\prime_j \sim P(\cdot | s_j, \bm{a}_j)$, $\bm{a}_j \sim \bm{\pi}_E(\cdot | s_j)$ and $M$ is the dataset size. Given $\tau_E$, we can get the empirical transition model $\hat{P}$ and the empirical expert joint policy $\hat{\bm{\pi}}_E$, thereby obtaining an approximate version $\hat{\mathcal{B}} = (\hat{\mathcal{M}}, \hat{\bm{\pi}}_E)$ of the MIRL problem $\mathcal{B}$. Here, $\hat{\mathcal{M}}$ denotes the empirical MG\textbackslash R. The feasible reward set $\mathcal{R}_{\hat{\mathcal{B}}}$ obtained through $\hat{\mathcal{B}}$ is essentially an estimate of the true feasible reward set $\mathcal{R}_\mathcal{B}$, and there is inevitably some error between them. We extend the results of \cite{metelli2021provably} to the multi-agent context.

\begin{proposition}[Error Propagation]
\label{proposition1}
\textit{Given two MIRL problems $\mathcal{B}=(\mathcal{M}, \bm{\pi}_E)$ and $\hat{\mathcal{B}}=(\hat{\mathcal{M}}, \hat{\bm{\pi}}_E)$, for any $r \in \mathcal{R}_\mathcal{B}$ there exists $\hat{r} \in \mathcal{R}_{\hat{\mathcal{B}}}$ such that:
\begin{equation*}
\begin{aligned}
    |r(s,\bm{a})-\hat{r}(s,\bm{a})| \leq & 
    \zeta \mathbb{I}_{\{\bm{\pi}_E(a|s)=0\}}\mathbb{I}_{\{\hat{\bm{\pi}}_E(a|s) > 0\}} + \\ &\gamma\sum_{s^\prime} |V(s^\prime)(P(s^\prime|s,a)-\hat{P}(s^\prime|s,a))|.
\end{aligned}
\end{equation*}}
\end{proposition}

\subsection{Symmetry-Guided MIRL}

Our goal is to design an algorithm that leverages the symmetry of MAS to improve the sample efficiency of MIRL algorithms and recover more accurate reward functions. The symmetry of MAS can be modeled through MG. Based on MG\textbackslash R, we propose the \textit{$\mathcal{G}$-invariant Markov Game without reward} ($\mathcal{G}$-invariant MG\textbackslash R) for MIRL.

\begin{definition}[$\mathcal{G}$-invariant MG\textbackslash R]
\textit{An MG\textbackslash R is said to be $\mathcal{G}$-invariant if $\forall g \in \mathcal{G}$, it satisfies the following equations:
\begin{align*}
    \pi_i(a_i|s) &= \pi_i(K_g[a_i]|L_g[s]), \\
    P(s^{\prime}|s,\bm{a}) &= P(L_g[s^{\prime}]|L_g[s],K_g[\bm{a}]),
\end{align*}
where $L_g$ and $K_g$ represent the transformations in state space and action space, respectively.}
\end{definition}

This paper aims to improve sample efficiency by designing algorithms based on the properties of $\mathcal{G}$-invariant MG\textbackslash R. The equivalent perspective of recovering the reward function from expert demonstrations $\tau_E=\{(s_j,\bm{a}_j,s^\prime_j)\}_{j=1}^M$ based on $\mathcal{G}$-invariant MG\textbackslash R is to recover the reward function from an augmented expert demonstration $\tau_E^\mathcal{G}=\{(gs_j,g\bm{a},gs^\prime_j)|g\in\mathcal{G},(s_j,\bm{a}_j,s_j)\in\tau_E\}_{j=1}^M$. Then, we can formally define the Symmetry-Guided MIRL problem as follow.

\begin{definition}[Symmetry-Guided MIRL Problem]
\label{definition2}
\textit{Given an empirical MIRL problem $\hat{\mathcal{B}}=(\hat{\mathcal{M}}, \hat{\bm{\pi}}_E)$, a Symmetry-Guided MIRL problem is a pair $\hat{\mathcal{B}}^\mathcal{G}=(\hat{\mathcal{M}}^\mathcal{G}, \hat{\bm{\pi}}_E^{\mathcal{G}})$, where $\hat{\mathcal{M}}^\mathcal{G}$ is the empirical $\mathcal{G}$-invariant MG\textbackslash R and $\hat{\bm{\pi}}_E^{\mathcal{G}}$ is the empirical expert joint policy, which are formed from the augmented expert demonstrations $\tau_E^{\mathcal{G}}$. The recovered reward set through $\hat{\mathcal{B}}^\mathcal{G}$ is denoted as $\mathcal{R}_{\hat{\mathcal{B}}^\mathcal{G}}$.}
\end{definition}

Furthermore, by combining Proposition \ref{proposition1} with Definition \ref{definition2}, we can conclude that the Symmetry-Guided MIRL problem can recover a more accurate reward function, as shown in Fig. \ref{symmetry}. We formally state this as the following proposition.

\begin{proposition}[Symmetry-Guided MIRL Improvement]
\label{proposition2}
\textit{Given the empirical MIRL problem $\hat{\mathcal{B}}=(\hat{\mathcal{M}}, \hat{\bm{\pi}}_E)$ and $\hat{B}^\mathcal{G}=(\hat{\mathcal{M}}^\mathcal{G}, \hat{\bm{\pi}}_E^\mathcal{G})$, the upper bound of error propagation satisfies the following relation:
\begin{equation*}
    |r(s,\bm{a})-\hat{r}(s, \bm{a})|^U \geq |r(s, \bm{a}) - \hat{r}^\mathcal{G}(s, \bm{a})|^U,
\end{equation*}
where $\hat{r}^\mathcal{G} \in \mathcal{R}_{\hat{\mathcal{B}}^\mathcal{G}}$ and $U$ denotes upper bound.}
\end{proposition}

From Proposition \ref{proposition2}, we can conclude that the error propagation with symmetry guidance in the worst case is always no greater than the error propagation in the worst case without symmetry guidance. In other words, Symmetry-Guided MIRL corresponds to a more accurate feasible reward set. All proofs are provided in the Appendix.

\begin{figure}[!t]
\centering
\includegraphics[width=\columnwidth]{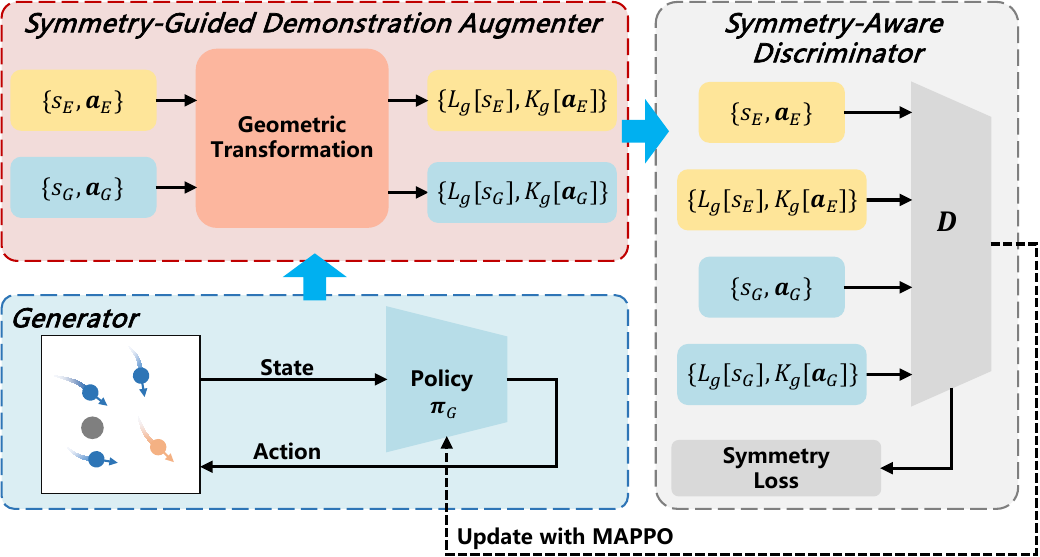} 
\caption{Overview of the proposed Symmetry-Guided Framework (SGF). The framework consists of two main modules. The Symmetry-Guided Demonstration Augmenter performs geometric transformations to augment the given data. The Symmetry-Aware Discriminator is used to enhance the ability to identify symmetric data. The subscript $G$ denotes entities related to the generator, while the subscript $E$ denotes entities related to the expert demonstrations.
}
\label{framework}
\end{figure}

\section{Symmetry-Guided Framework}
Based on Section \ref{sec4}, we propose a Symmetry-Guided Framework (SGF) for multi-agent adversarial IRL to improve the sample efficiency. It is worth noting that the theory presented in this paper is not limited to these algorithms and can also be applied to other MIRL algorithms. The SGF consists of two primary components: Symmetry-Guided Demonstration Augmenter (SGDA) and Symmetry-Aware Discriminator (SAD). The overall architecture is depicted in Fig. \ref{framework} and the training process is shown in Algorithm \ref{alg1}.

\subsection{Symmetry-Guided Demonstration Augmenter}

Based on Definition \ref{definition2}, the most direct way to leverage symmetry in MIRL is through data augmentation. To overcome the limitation of insufficient expert demonstrations in multi-agent adversarial IRL, this paper proposes a Symmetry-Guided Demonstration Augmenter that utilizes the symmetry in multi-robot systems to augment the dataset.

In the context of multi-robot systems, symmetry is primarily manifested through equivariance and invariance with respect to rotations and reflections. These geometric transformations are represented by the orthogonal group $\mathcal{O}(n)$, defined as follows:
\begin{equation*}
    \mathcal{O}(n) = \{Z \in \mathbb{R}^{n \times n} : Z^T Z = I\},
\end{equation*}
where $Z^T$ denotes the transpose of matrix $Z$, and $I$ represents the identity matrix. In this paper, we focus on a discrete subgroup of the $\mathcal{O}(n)$ group, specifically  the dihedral group $\mathcal{D}_n$, which is defined as follows:
\begin{equation*}
    \mathcal{D}_n = \{R, S \mid R^n = I, S^2 = I, SRS = R^{-1}\},
\end{equation*}
where $R$ and $S$ signify rotation and reflection operations, respectively. The $\mathcal{D}_n$ group includes $2n$ elements, consisting of $n$ rotational and $n$ reflective transformations. In the subsequent section, we explore the application of data augmentation based on the $\mathcal{D}_n$ group within the $\mathcal{G}$-invariant MG\textbackslash R framework.

\subsubsection{State Transformation}

In multi-robot systems, the environmental features can be of various types. Some features exhibit equivariance to geometric transformations, denoted as $s_{equ}$, while others show invariance, remaining unaffected by these transformations, denoted as $s_{inv}$. We represent the environmental state as the concatenation of these two types of features, $s = [s_{equ}, s_{inv}] \in \mathcal{S}$. For $s_{equ}$, our module applies predefined geometric transformations $g \in \mathcal{D}_n$ to simulate potential real-world changes in orientation and position. For example, features representing robot coordinates are rotated across various angles to generate new feasible states. On the other hand, since $s_{inv}$ is unaffected by geometric transformations, it is preserved in its original form during data augmentation. Invariance typically applies to features such as the internal states of robots and distances. In general, the impact of the geometric transformations $g \in \mathcal{D}_n$ on $s \in \mathcal{S}$ can be expressed as $L_g[s] = [L_g[s_{equ}], s_{inv}]$.

\subsubsection{Action Transformation}

According to the $\mathcal{G}$-invariant MG\textbackslash R, when geometric transformations act on the environmental state, the agents' actions are also subjected to corresponding transformations. Similar to the representation of environmental states, the agents' actions are categorized into equivariant actions $\bm{a}_{equ}$ and invariant actions $\bm{a}_{inv}$, denoted as $\bm{a} = [\bm{a}_{equ}, \bm{a}_{inv}] \in \mathcal{A}$. For $\bm{a}_{equ}$, typically directional actions, the same geometric transformations $g \in \mathcal{D}_n$ as the state are applied to ensure consistency between the state and the action. During augmentation, invariant actions $\bm{a}_{inv}$ are treated as constants since their execution is independent of the geometric properties of the environment. Therefore, the geometric transformation $g \in \mathcal{D}_n$ acts on the actions $\bm{a} \in \mathcal{A}$ as $K_g[\bm{a}] = [K_g[\bm{a}_{equ}], \bm{a}_{inv}]$.

The SGDA leverages the inherent symmetry of multi-robot systems to expand the scope of given demonstrations. This module enables the demonstrations to cover a broader range of environmental states, effectively increasing the number of available demonstrations without the need for additional collection of costly demonstrations. Based on Proposition \ref{proposition2}, augmented demonstrations through SGDA can recover a more accurate reward function, thus improving the sample efficiency of MIRL algorithms.

\begin{algorithm}[!t]
\caption{Symmetry-Guided Framework.}
\label{alg1}
\begin{algorithmic}[1] 
\STATE \textbf{Input}: Number of agents $N$, expert demonstrations $\tau_E$, geometric transformations $g \in \mathcal{D}_n$.
\STATE \textbf{Initialize}: Generator network $\bm{\pi}_{\bm{\theta}}$, discriminators $D_{\bm{\omega}}$.
\STATE Perform data augmentation on $\tau_E$: $\tau_E \leftarrow \tau_E \cup g\tau_E$.
\REPEAT
    \STATE Sample generator trajectories $\tau_{G}$ from $\bm{\pi}_{\theta}$.
    \STATE Sample state-action pairs $\mathcal{X}_E$, $\mathcal{X}_G$ from $\tau_E$, $\tau_G$.
    \FOR{$i=1,\dots,N$}
        \STATE Perform data augmentation on $\mathcal{X}_E$, $\mathcal{X}_G$ through $g$.
        \STATE Update $\omega_i$ to minimize the objective in Eq. \eqref{loss3}.
    \ENDFOR
    \FOR{$i=1,\dots,N$}
        \STATE Update $\theta_i$ with the feedback of $D_{\omega_i}$ by MAPPO.
    \ENDFOR
\UNTIL{Convergence}
\STATE \textbf{Output}: Learned policy $\bm{\pi}_{\bm{\theta}}$.
\end{algorithmic}
\end{algorithm}

\begin{figure*}[!t]
     \centering
     \begin{subfigure}[b]{0.32\textwidth}
         \centering
         \includegraphics[width=\textwidth]{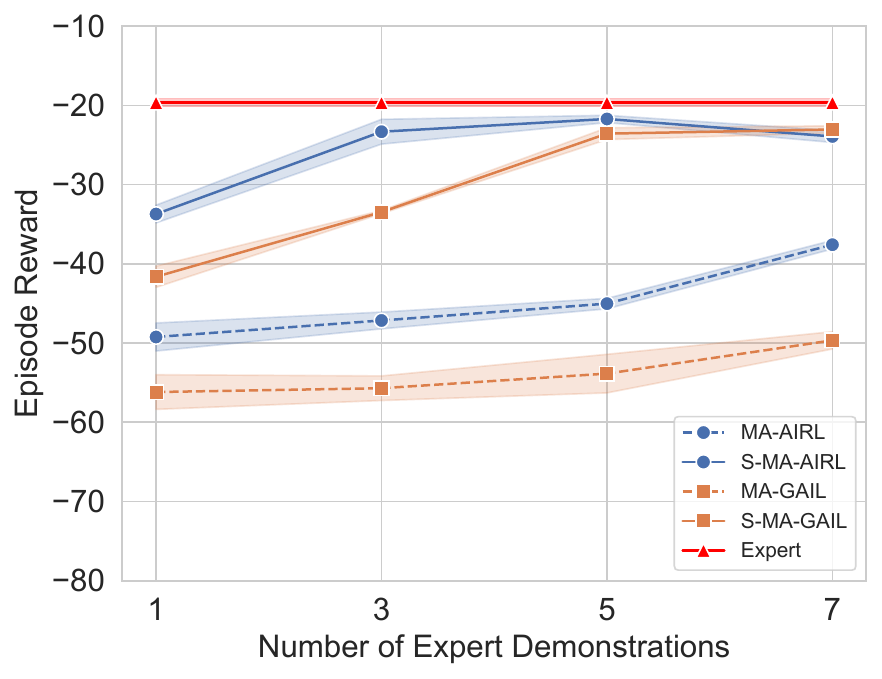}
         \caption{Rendezvous.}\label{expa}
     \end{subfigure}
     \hfill
     \begin{subfigure}[b]{0.32\textwidth}
         \centering
         \includegraphics[width=\textwidth]{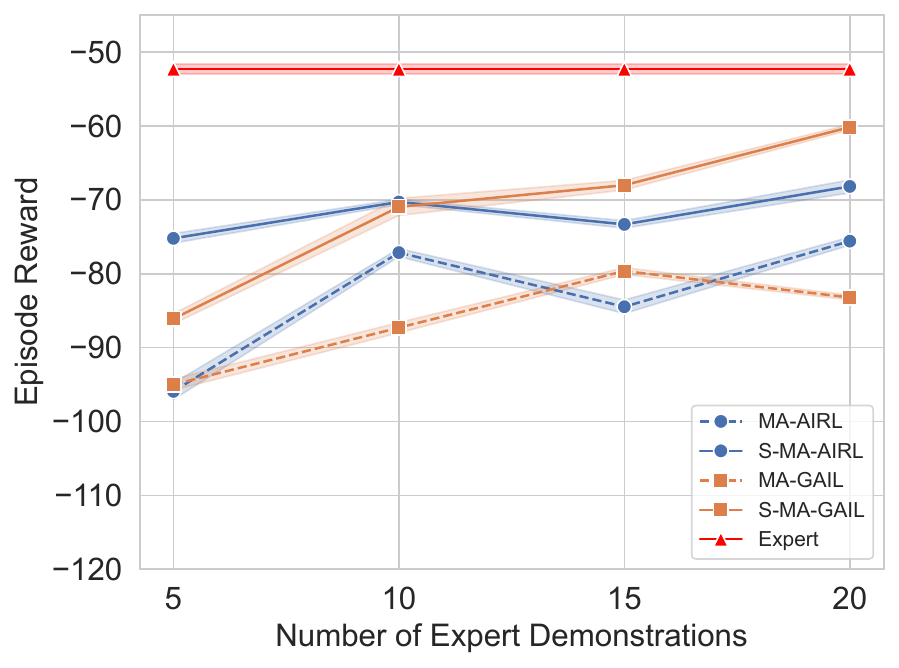}
         \caption{Pursuit.}\label{expb}
     \end{subfigure}
     \hfill
     \begin{subfigure}[b]{0.32\textwidth}
         \centering
         \includegraphics[width=\textwidth]{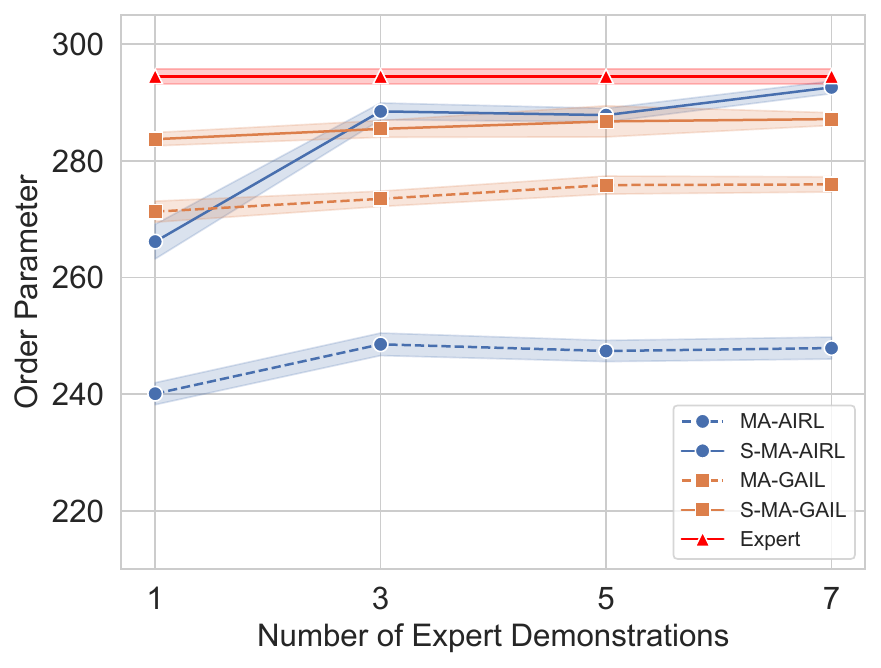}
         \caption{Vicsek.}\label{expc}
     \end{subfigure}
     \caption{Average true rewards across three tasks. Experiments are conducted with four different numbers of expert demonstrations. Each experiment is repeated under five different random seeds to ensure the reliability of the results.}\label{fig4}
\end{figure*}

\subsection{Symmetry-Aware Discriminator}
In multi-agent adversarial IRL, the policy is trained using a generative adversarial method through expert demonstrations $\tau_E=\{(s_j,\bm{a}_j,s^\prime_j)\}_{j=1}^M$. The generator is a specific MARL algorithm, with its policy denoted as $\bm{\pi}_G$. The role of discriminators is to distinguish the behaviors produced by $\bm{\pi}_G$ and those of the expert as accurately as possible. The objective of discriminators is to minimize the cross-entropy between expert demonstrations and generator samples:
\begin{equation}
\begin{aligned}
    \mathcal{L}({\bm{\omega}})=&-\mathbb{E}_{\tau_E}\left[\sum_i^N \log D_{\omega_i}\left(s,a_i,s^{\prime}\right)\right]- \\
    &\mathbb{E}_{\bm{\pi}_G}\left[\sum_i^N \log \left(1-D_{\omega_i}(s,a_i,s^{\prime})\right)\right], \label{loss1}
\end{aligned}
\end{equation}
where $\bm{\omega}=\{\omega_1, \dots, \omega_N\}$ are parameters of discriminators. Due to the inherent symmetry of multi-robot systems, the discriminator must be able to effectively distinguish between expert demonstrations and generated samples that have undergone geometric transformations. This capability is crucial for providing accurate feedback to the generative model. However, when geometric transformations are applied to expert demonstrations and generated samples, the optimization direction induced by the loss function \eqref{loss1} fails to guide the discriminator to distinguish whether transformed samples originate from expert demonstrations or are generated. To appropriately utilize symmetry, we propose a symmetry-aware discriminator, with the new loss defined as follows:
\begin{equation}
    \begin{aligned}
        \mathcal{L}_S(\bm{\omega})=&-\mathbb{E}_{\tau_E}\left[\sum_i^N \log D_{\omega_i}\left(L_g[s], K_g[a_i], L_g[s^{\prime}]\right)\right] -\\
    &\mathbb{E}_{\bm{\pi}_G}\left[\sum_i^N \log \left(1-D_{\omega_i}(L_g[s], K_g[a_i], L_g[s^{\prime}])\right)\right], \label{loss2}
    \end{aligned}
\end{equation}
This function aims to minimize the cross-entropy between the transformed expert demonstrations and generated samples. This loss helps the discriminator better identify data that has been symmetrically augmented, providing more precise feedback to the generator. Discriminators of our SGF framework minimize the following objective:
\begin{equation}
    \mathcal{L}_{SGF} = \mathcal{L}(\bm{\omega})+\mathcal{L}_S(\bm{\omega}). \label{loss3}
\end{equation}


\section{Experiments}

\subsection{Experimental Settings}
\label{sec_exp_setting}
The proposed framework is applied to the mainstream multi-agent adversarial IRL algorithms, including MA-GAIL \cite{song2018multi} and MA-AIRL \cite{yu2019multi}. We use the transformations contained in the $\mathcal{D}_4$ group as the augmentation methods in SGDA. We denote algorithms that use SGF by the prefix ``S-''. Experiments are conducted across three multi-agent continuous cooperative tasks including Rendezvous, Pursuit, and Vicsek. Each environment incorporates an underlying reward function or evaluation metric that facilitates the evaluation of the generated policies. We set the number of expert demonstrations based on the size of the state space. Tasks with larger state spaces require more expert demonstrations. The bold numbers in all tables indicate the experimental results using SGF.

\subsubsection{Rendezvous}
As shown in Fig. \ref{env1}, this task requires agents to gather autonomously without a predefined target point to minimize the distances between each other. We train the expert agents based on SOTA MARL algorithms (like MAPPO \cite{yu2022surprising}) with the underlying reward. The trained policies are then used to generate expert demonstrations.
\subsubsection{Pursuit}
As shown in Fig. \ref{env2}, multiple predators attempt to chase the prey that adopts a Voronoi strategy \cite{zhou2016cooperative} and has a speed 1.5 times faster than that of the predators. Expert policies continue to be trained using the underlying reward with SOTA MARL algorithms.

\subsubsection{Vicsek}
The Vicsek model \cite{vicsek1995novel} is a classic collective physics model where agents attempt to align their directions of movement as shown in Fig. \ref{env3}. In this task, we utilize the Vicsek model to generate expert demonstrations in randomly initialized scenarios. The effectiveness of the generated policies is evaluated using order parameters.

\begin{figure}[!t]
  \centering
  \begin{subfigure}[b]{0.32\columnwidth}
    \includegraphics[width=\linewidth]{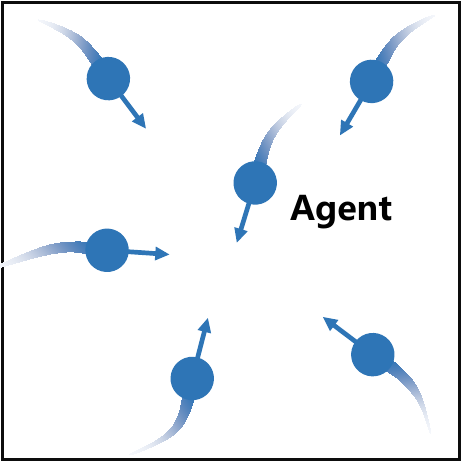}
    \caption{Rendezvous.}
    \label{env1}
  \end{subfigure}
  \begin{subfigure}[b]{0.32\columnwidth}
    \includegraphics[width=\linewidth]{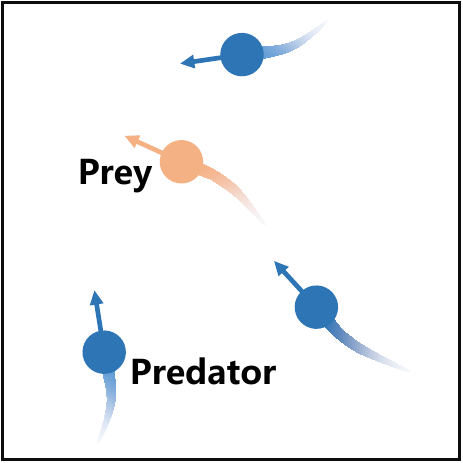}
    \caption{Pursuit.}
    \label{env2}
  \end{subfigure}
  \begin{subfigure}[b]{0.32\columnwidth}
    \includegraphics[width=\linewidth]{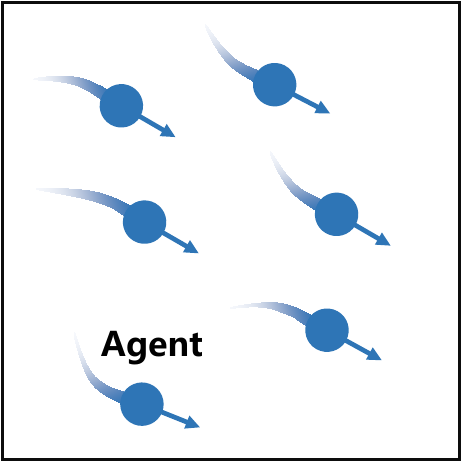}
    \caption{Vicsek.}
    \label{env3}
  \end{subfigure}
  \caption{The simulated tasks considered in this paper.}
  \label{env}
\end{figure}

\subsection{Main Results}
In this section, we present and analyze the main experimental results under the settings described in Section \ref{sec_exp_setting}. Experiments are conducted on all tasks with varying numbers of expert demonstrations, with the number of agents fixed at 10. All results are the average of five different random seeds, as shown in Fig. \ref{fig4}.

\subsubsection{Rendezvous}
Experiments are conducted with $M\in\{100,300,500,700\}$. We evaluate the performance of baselines (MA-AIRL and MA-GAIL) compared to those enhanced by SGF (S-MA-AIRL and S-MA-GAIL). As shown in Fig. \ref{expa}, the symmetry-guided algorithms consistently outperform the baselines in terms of convergence rewards across all demonstration sizes tested. SGF enables algorithms to achieve performance comparable to algorithms that use a larger number of demonstrations, even when utilizing fewer demonstrations. These results highlight the significant impact of SGF in improving sample efficiency and policy learning.

\subsubsection{Pursuit}
Since the presence of prey increases the state space of the environment, we use more expert demonstrations, specifically $M \in \{500,1000,1500,2000\}$. As depicted in Fig. \ref{expb}, as the number of expert demonstrations increased, the baselines exhibited instability. However, the algorithms enhanced with the SGF mitigate this instability and effectively improve the underlying rewards in all demonstration numbers. Moreover, SGF allows algorithms to learn more effective policies with fewer samples. These results demonstrate the effectiveness of SGF in improving algorithm performance in complex tasks.

\subsubsection{Vicsek}
We conduct experiments with $M \in \{100, 300, 500, 700\}$. As shown in Fig. \ref{expc}, the order parameters of the baseline models gradually increase as the number of expert demonstrations increases, with the order parameter remaining at a relatively low level. In contrast, algorithms that use SGF exhibit a significant increase in the order parameter, reaching levels close to those of the expert. We attribute this improvement to the use of symmetry, which expands the range of expert demonstrations covering environmental states. This enables the learned policy to handle a wider variety of situations more effectively. This result further validates that our framework can effectively improve both learning efficiency and policy performance.

\begin{table}[!t]
    \caption{Ablation study on Symmetry-Aware Discriminator (SAD), as expressed by the mean convergence rewards. ``Normal'' refers to discriminators that do not use symmetry. }
    \centering
    \setlength{\tabcolsep}{2pt}
    \resizebox{0.88\columnwidth}{!}{
    \begin{tabular}{cc|c|c|c|c}
    \toprule
    \multirow{2}{*}{Task} & Expert & \multicolumn{2}{c|}{S-MA-AIRL} & \multicolumn{2}{c}{S-MA-GAIL} \\
    & Num & SAD & Normal & SAD & Normal \\
    \midrule
    \multirow{4}{*}{Rendezvous} 
    & 100 & $\textbf{-33.7} \pm 1.1$ & $-43.4 \pm 1.0$ & $\textbf{-41.6} \pm 1.3$ & $-42.2 \pm 0.8$  \\

    & 300 & $\textbf{-23.1} \pm 0.7$ & $-29.4 \pm 0.7$ & $\textbf{-33.8} \pm 0.3$ & $-40.8 \pm 0.8$  \\
    & 500 & $\textbf{-21.7} \pm 0.5$ & $-33.2 \pm 1.2$ & $\textbf{-23.6} \pm 0.7$ & $-29.0 \pm 1.3$ \\
    & 700 & $\textbf{-23.9} \pm 0.8$ & $-35.9 \pm 0.9$ & $\textbf{-23.5} \pm 0.6$ & $-26.9 \pm 0.8$ \\
    \midrule
    \multirow{4}{*}{Pursuit} 
    & 500 & $\textbf{-75.2} \pm 0.7$ & $-80.1 \pm 1.0$ & $\textbf{-86.0} \pm 0.7$ & $-89.3 \pm 0.8$  \\
    & 1000 & $\textbf{-70.3} \pm 0.4$ & $-77.8 \pm 1.3$ & $\textbf{-70.9} \pm 1.2$ & $-78.7 \pm 0.9$  \\
    & 1500 & $\textbf{-73.3} \pm 0.5$ & $-76.7 \pm 0.8$ & $\textbf{-68.0} \pm 0.7$ & $-76.9 \pm 0.6$  \\
    & 2000 & $\textbf{-68.1} \pm 0.9$ & $-72.9 \pm 0.6$ & $\textbf{-60.1} \pm 0.4$ & $-68.1 \pm 1.1$  \\
    \midrule
    \multirow{4}{*}{Vicsek} 
    & 100 & $\textbf{266.1} \pm 2.9$ & $235.3 \pm 2.3$ & $\textbf{283.7} \pm 1.1$ & $280.1 \pm 1.2$ \\
    & 300 & $\textbf{286.2} \pm 2.0$ & $276.5 \pm 2.5$ & $\textbf{285.5} \pm 1.4$ & $281.9 \pm 1.4$  \\
    & 500 & $\textbf{287.8} \pm 1.2$ & $267.5 \pm 2.8$ & $\textbf{286.7} \pm 2.6$ & $272.8 \pm 1.8$  \\
    & 700 & $\textbf{292.5} \pm 1.1$ & $284.8 \pm 2.4$ & $\textbf{287.1} \pm 1.1$ & $274.6 \pm 2.0$  \\
    \bottomrule
    \end{tabular}}
    \label{tab:abl}
\end{table}

\begin{figure}[!t]
  \centering
  \begin{subfigure}[t]{0.3343\columnwidth}
    \includegraphics[width=\linewidth]{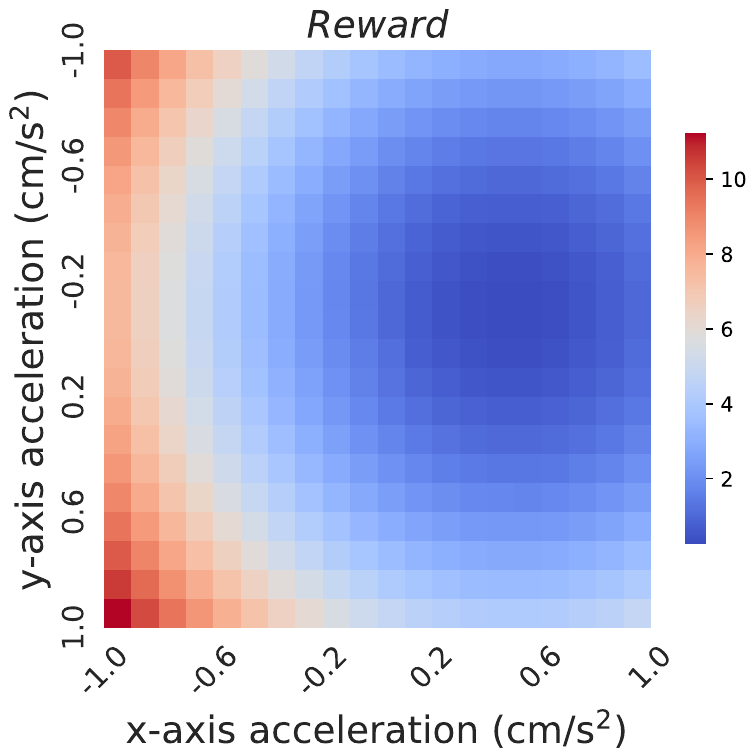}
    \caption{S-MA-AIRL}
    \label{reward1}
  \end{subfigure}
  \begin{subfigure}[t]{0.2913\columnwidth}
    \includegraphics[width=\linewidth]{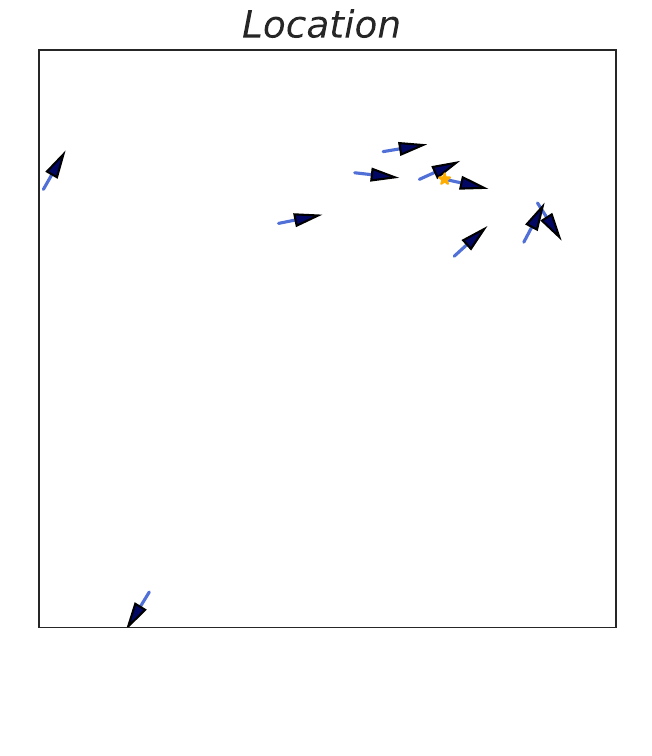}
    \caption{Agent Location}
    \label{location}
  \end{subfigure}
  \begin{subfigure}[t]{0.3343\columnwidth}
    \includegraphics[width=\linewidth]{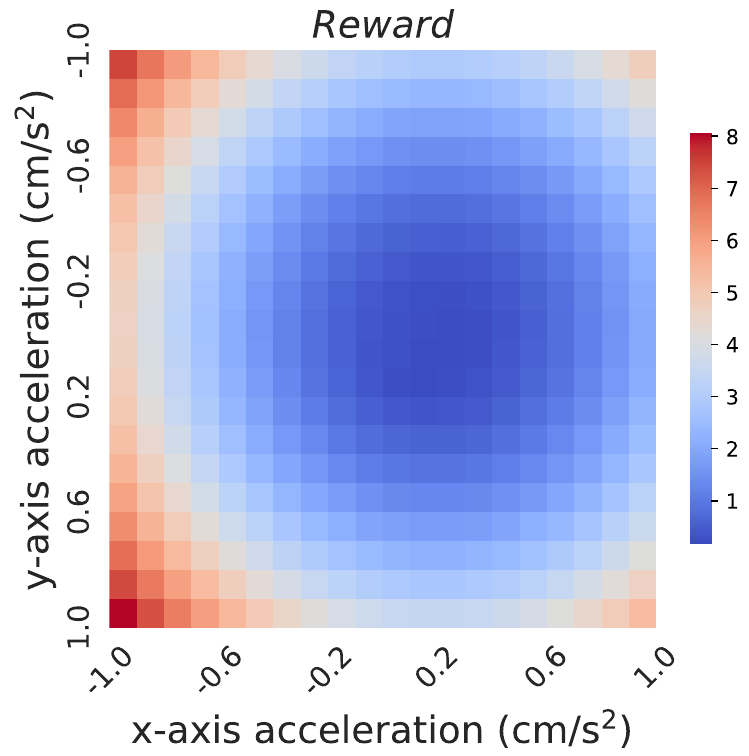}
    \caption{MA-AIRL}
    \label{reward2}
  \end{subfigure}
  \caption{Visualization of rewards for the yellow-starred agent under different actions in the Rendezvous task. The triangle indicates the agent's direction.}
  \label{reward}
\end{figure}

\subsection{Analysis of the Recovered Reward Function}

We conduct a visual analysis of the reward functions recovered by the algorithms, comparing S-MA-AIRL and MA-AIRL. Taking the Rendezvous task as an example, we randomly select a state at a specific time and evaluate the rewards for a random agent under different actions (Fig. \ref{location}). The actions are the accelerations on the $x$ and $y$ axes. The rewards recovered by S-MA-AIRL (Fig. \ref{reward1}) indicate that the agent should accelerate towards the upper-left or lower-left directions, i.e., decelerating in the current direction. Given its relatively rightward position in the swarm, this reward is reasonable. In contrast, the rewards recovered by MA-AIRL (Fig. \ref{reward2}) show higher rewards for accelerating towards the lower-right or upper-right, which is detrimental to completing the Rendezvous task, making the recovered rewards inaccurate. This result validates our theory that incorporating symmetry leads to a more accurate reward recovery.

\subsection{Ablation Study}
This section presents the experimental results testing the impact of SAD across three tasks, with detailed outcomes shown in Table \ref{tab:abl}. The number of agents used in the experiments is 10. Across all tasks, with different numbers of expert demonstrations, the use of SAD consistently improves the performance of the generated policies. This validates that using SAD helps better identify symmetric data, thereby providing more informative feedback to the generator. These results highlight the importance of incorporating symmetry into the discriminator.

\begin{table}[!t]
    \caption{This table displays the impact of varying numbers of agents on different tasks, as expressed by the mean convergence rewards of the models.}
    \centering
    \setlength{\tabcolsep}{1pt}
    \resizebox{1.0\columnwidth}{!}{
    \begin{tabular}{cc|c|cc|cc}
    \toprule
    \multirow{2}{*}{Task} & Agent & \multicolumn{5}{c}{Model} \\
    & Num & Expert & MA-AIRL & S-MA-AIRL & MA-GAIL & S-MA-GAIL \\
    \midrule
    \multirow{4}{*}{Rendezvous} 
    & 5 & $-22.3 \pm 2.2$ & $-35.8 \pm 1.6$ & $\textbf{-28.3} \pm 1.2$ & $-42.8 \pm 1.4$ & $\textbf{-29.8} \pm 1.3$ \\
    & 10 & $-19.9 \pm 0.3$ & $-47.4 \pm 1.3$ & $\textbf{-23.1} \pm 0.7$ & $-57.2 \pm 2.0$ & $\textbf{-33.8} \pm 0.4$ \\
    & 15 & $-23.9 \pm 0.3$ & $-44.9 \pm 0.9$ & $\textbf{-29.5} \pm 0.9$ & $-59.7 \pm 1.1$ & $\textbf{-46.1} \pm 1.1$ \\
    & 20 & $-24.2 \pm 0.4$ & $-45.7 \pm 0.9$ & $\textbf{-26.0} \pm 0.4$ & $-65.1 \pm 1.2$ & $\textbf{-29.0} \pm 0.7$ \\
    \midrule
    \multirow{4}{*}{Pursuit} 
    & 5 & $-74.7 \pm 0.9$ & $-81.8 \pm 1.2$ & $\textbf{-79.9} \pm 0.6$ & $-98.6 \pm 1.1$ &  $\textbf{-84.0} \pm 0.7$\\
    & 10 & $-52.3 \pm 0.4$ & $-77.1 \pm 0.5$ &  $\textbf{-70.6} \pm 0.4$ & $-87.3 \pm 0.7$ & $\textbf{-71.0} \pm 1.2$ \\
    & 15 & $-51.6 \pm 0.3$ & $-82.9 \pm 1.0$ & $\textbf{-60.3} \pm 0.3$ & $-84.0 \pm 0.4$ & $\textbf{-67.0} \pm 0.7$ \\
    & 20 & $-51.2 \pm 0.5$ & $-70.5 \pm 1.2$ & $\textbf{-61.5} \pm 1.0$ & $-84.1 \pm 1.5$ & $\textbf{-69.9} \pm 0.5$ \\
    \midrule
    \multirow{4}{*}{Vicsek} 
    & 5 & $295.3 \pm 1.9$ & $269.8 \pm 1.4$ & $\textbf{270.5} \pm 1.2$ & $279.1 \pm 0.8$ & $\textbf{289.5} \pm 0.4$ \\
    & 10 & $293.5 \pm 2.4$ & $247.5 \pm 1.1$ & $ \textbf{286.2} \pm 2.0$ & $271.3 \pm 1.8$ & $\textbf{285.5} \pm 1.4$ \\
    & 15 & $292.4 \pm 2.7$ & $244.2 \pm 1.7$ & $\textbf{268.9} \pm 2.5$ & $268.3 \pm 1.1$ & $\textbf{284.5} \pm 0.8$ \\
    & 20 & $291.1 \pm 1.6$ & $240.5 \pm 5.7$ & $\textbf{255.4} \pm 1.5$ & $262.5 \pm 0.9$ & $\textbf{286.1} \pm 0.7$ \\
    \bottomrule
    \end{tabular}}
    \label{tab:num}
\end{table}

\subsection{The Impact of Different Numbers of Agents}
In this section, we explore the impact of varying the number of agents (5, 10, 15, 20) on the performance of algorithms. Experiments are conducted with 300 expert demonstrations for the Rendezvous and Vicsek and 1000 for the Pursuit. The results are presented in Table \ref{tab:num}. The findings show that algorithms using SGF consistently outperform the baseline algorithms, regardless of the number of agents. As the number of agents increases, the performance of the baseline algorithms deteriorates. However, SGF effectively mitigates this degradation, maintaining strong performance even with a large number of agents.


\begin{figure}[!t]
  \centering
  \begin{subfigure}[b]{0.49\columnwidth}
    \includegraphics[width=\linewidth]{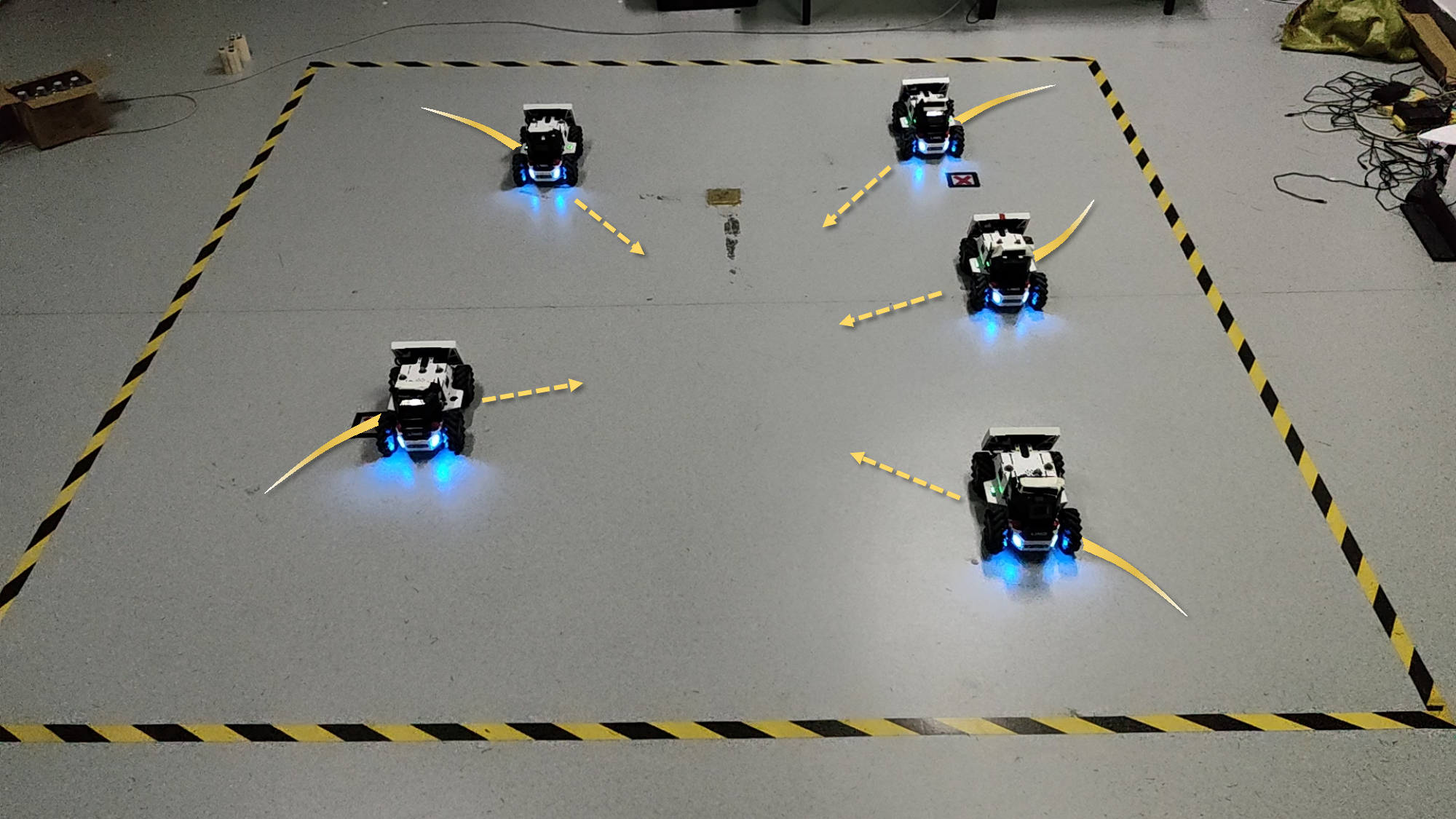}
    \caption{Rendezvous.}
    \label{phy1}
  \end{subfigure}
  \begin{subfigure}[b]{0.49\columnwidth}
    \includegraphics[width=\linewidth]{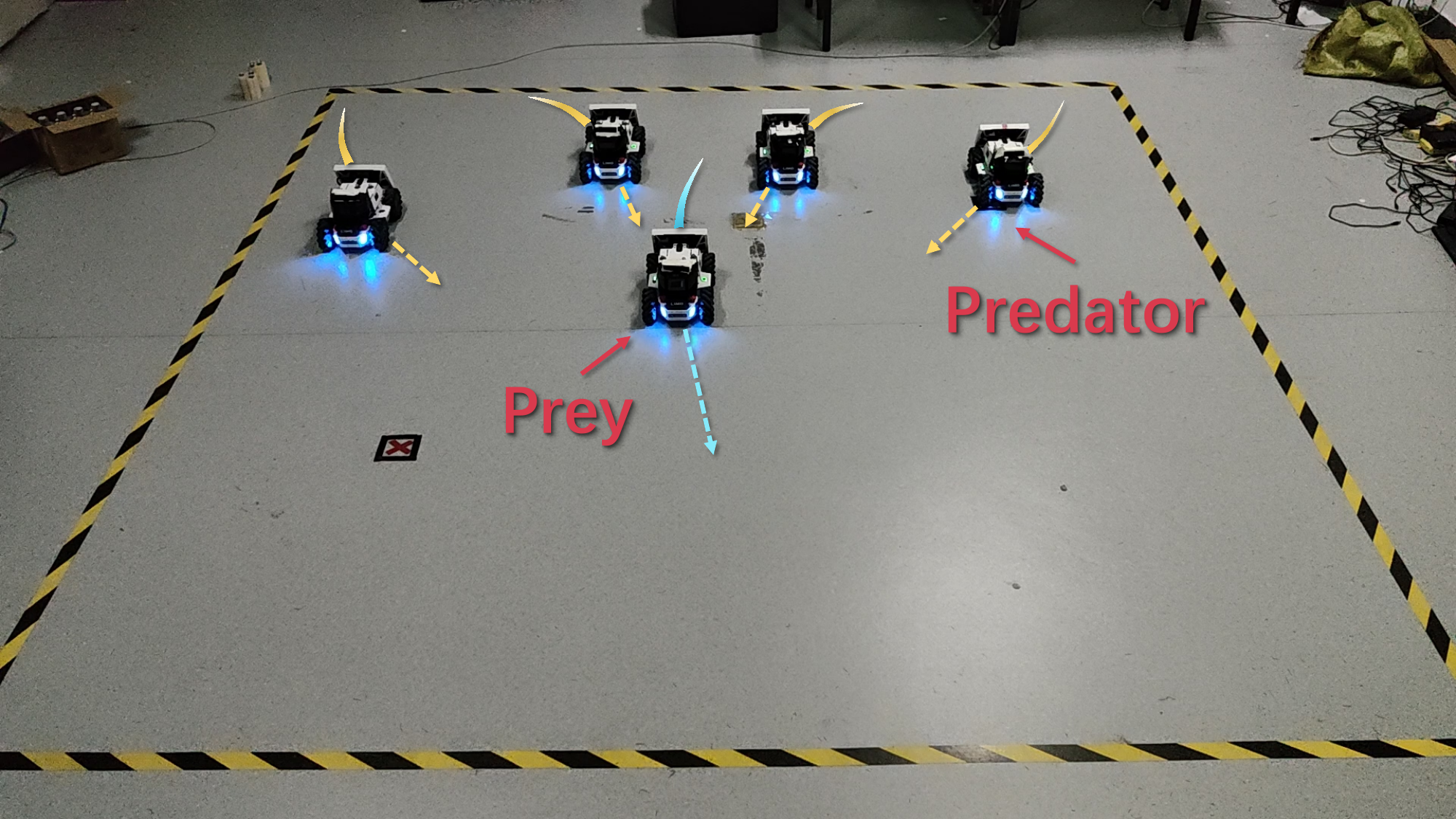}
    \caption{Pursuit.}
    \label{phy2}
  \end{subfigure}
  \caption{Demonstration of the real-world scenarios.}
  \label{phy}
\end{figure}

\begin{table}[!t]
    \centering
    \setlength{\tabcolsep}{10pt}
    \caption{Performance results of the real-world tasks.}
    
    \begin{tabular}{c|cc}
    \toprule
    Task & MA-AIRL & S-MA-AIRL \\
    \midrule
    Rendezvous & 668.6 & \textbf{580.8} \\
    Pursuit & 1237.7 & \textbf{1023.8} \\
    \bottomrule
    \end{tabular}
    \label{tab:phy_res}
\end{table}

\subsection{Demonstration on Robots}
In this section, we deploy our algorithms in real-world environments to validate the effectiveness of SGF. Since Vicsek uses periodic boundary conditions that cannot be implemented in the real world, we focus on the other tasks instead. Specifically, we use Limo robots as agents, controlled via the Robot Operating System (ROS) system, with real-time environmental states captured by Nokov motion capture system. The setup of the task is shown in Fig. \ref{phy}. The Rendezvous involves 5 agents, while the Pursuit consists of 4 predators and 1 prey. Without loss of generality, we choose MA-AIRL and S-MA-AIRL for real-world validation, using cumulative distance as the performance metric. The results, as presented in Table \ref{tab:phy_res}, show that the algorithms using SGF complete tasks faster in both tasks, confirming the effectiveness of our method in real-world applications.

\section{Conclusions}
This paper aims to address the problem of sample inefficiency in MIRL. We provide the first proof that leveraging the inherent symmetry in MAS can recover more accurate reward functions. Based on this theory, we propose a novel framework for multi-agent adversarial IRL algorithms, termed the Symmetry-Guided Framework (SGF). Experimental results show that our framework significantly enhances the performance of the algorithms, achieving superior convergence rewards. Furthermore, the policies learned with the recovered rewards have been deployed in physical multi-robot systems, further validating the effectiveness of our framework in real-world scenarios.


\section*{Appendixes}
The $Q$-function and $V$-function of a policy $\bm{\pi}$ in MG $\mathcal{M} \cup r$ are defined as:
\begin{align*}
    Q^{\bm{\pi}}_{\mathcal{M} \cup r}(s,\bm{a})=&\gamma\sum_{s^\prime,\bm{a}^\prime}\bm{\pi}(\bm{a}^\prime|s^\prime) P(s^\prime | s,\bm{a}) Q^{\bm{\pi}}_{\mathcal{M} \cup r}(s^\prime,\bm{a}^\prime)\\
    &+r(s,\bm{a}), \\
    V^{\bm{\pi}}_{\mathcal{M} \cup r}(s)=&\sum_{\bm{a}} \bm{\pi}(\bm{a}|s)Q^{\bm{\pi}}_{\mathcal{M} \cup r}(s,\bm{a}).
\end{align*}

If policy $\bm{\pi}_E$ is the optimal policy under reward $r$, then we can derive the following equations:
\begin{align}
    Q^{\bm{\pi}_E}_{\mathcal{M} \cup r}(s,\bm{a}) - V^{\bm{\pi}}_{\mathcal{M} \cup r}(s) = 0, \quad if \ \bm{\pi}_E(\bm{a}|s)>0, \label{eq4} \\
    Q^{\bm{\pi}_E}_{\mathcal{M} \cup r}(s,\bm{a}) - V^{\bm{\pi}}_{\mathcal{M} \cup r}(s) \leq 0, \quad if \ \bm{\pi}_E(\bm{a}|s)=0. \label{eq5}
\end{align}

\begin{lemma}
A $Q$-function satisfies the condition of equations \eqref{eq4} and \eqref{eq5} if and only if there exists $\zeta\in\mathbb{R}^{\mathcal{S}\times \mathcal{S}}_{\geq0}$ and $V\in\mathbb{R}^{\mathcal{S}}$ such that:
\begin{equation*}
    Q^{\bm{\pi}_E}_{\mathcal{M} \cup r}(s,\bm{a})=-\zeta\mathbb{I}_{\{\bm{\pi}_E(\bm{a}|s)=0\}} + V(s).
\end{equation*}
\end{lemma}

\begin{proof}
We first show that if $Q^{\bm{\pi}_E}_{\mathcal{M} \cup r}(s,\bm{a})$ has this form, then the equations \eqref{eq4} and \eqref{eq5} are satisfied, and then converse. Assuming $Q^{\bm{\pi}_E}_{\mathcal{M} \cup r}(s,\bm{a})=-\zeta\mathbb{I}_{\{\bm{\pi}_E(\bm{a}|s)=0\}} + V(s)$, then we can get $V^{\bm{\pi}_E}_{\mathcal{M} \cup r}(s)=\sum_{\bm{a}} \bm{\pi}_E(\bm{a}|s)Q^{\bm{\pi}_E}_{\mathcal{M} \cup r}(s,\bm{a})=V(s)$. If $\bm{\pi}_E(\bm{a}|s)>0$, then $Q^{\bm{\pi}_E}_{\mathcal{M} \cup r}(s,\bm{a})=V^{\bm{\pi}_E}_{\mathcal{M} \cup r}(s)$, which satisfies equation \eqref{eq4}. If $\bm{\pi}_E(\bm{a}|s)=0$, $Q^{\bm{\pi}_E}_{\mathcal{M} \cup r}(s,\bm{a})=V^{\bm{\pi}_E}_{\mathcal{M} \cup r}(s)-\zeta \leq V^{\bm{\pi}_E}_{\mathcal{M} \cup r}(s)$, which satisfies equation \eqref{eq5}. For the converse, assume equations \eqref{eq4} and \eqref{eq5} hold. We take $V(s)=V^{\bm{\pi}_E}_{\mathcal{M} \cup r}(s)$ and $\zeta=V^{\bm{\pi}_E}_{\mathcal{M} \cup r}(s)-Q^{\bm{\pi}_E}_{\mathcal{M} \cup r}(s,\bm{a}) \geq 0$.
\end{proof}

Then we can give the proof of Lemma \ref{lemma1} as follows.

\begin{proof}
From the definition of $Q$-function and $V$-function, we have:
\begin{equation*}
    Q^{\bm{\pi}}_{\mathcal{M} \cup r}(s,\bm{a})=r(s,\bm{a})+\sum\nolimits_{s^\prime} P(s^\prime | s,\bm{a}) V(s^\prime)
\end{equation*}
Then, we can get:
\begin{align*}
    r(s,\bm{a})&=Q^{\bm{\pi}}_{\mathcal{M} \cup r}(s,\bm{a})-\gamma\sum\nolimits_{s^\prime} P(s^\prime | s,\bm{a}) V(s^\prime) \\
    &=-\zeta\mathbb{I}_{\{\bm{\pi}_E(\bm{a}|s)=0\}} + V(s)-\gamma\sum\nolimits_{s^\prime} P(s^\prime | s,\bm{a}) V(s^\prime)
\end{align*}
\end{proof}

Next, we provide the proof of Proposition \ref{proposition1}.

\begin{proof}
We can rewrite the reward functions $r$ and $\hat{r}$ as:
\begin{align*}
    r(s, \bm{a})&=-\zeta\mathbb{I}_{\{\bm{\pi}_E(\bm{a}|s)=0\}} + V(s)-\gamma\sum\nolimits_{s^\prime} P(s^\prime | s,\bm{a}) V(s^\prime),\\
    \hat{r}(s,\bm{a})&=-\hat{\zeta}\mathbb{I}_{\{\hat{\bm{\pi}}_E(\bm{a}|s)=0\}} + \hat{V}(s)-\gamma\sum\nolimits_{s^\prime} \hat{P}(s^\prime | s,\bm{a}) \hat{V}(s^\prime).
\end{align*}

Without loss of generality, we can choose $V=\hat{V}$ and $\hat{\zeta}=\zeta\mathbb{I}_{\{\bm{\pi}_E(\bm{a}|s)=0\}}$:
\begin{align*}
    r(s, \bm{a})-\hat{r}(s,\bm{a})
    =& \zeta \mathbb{I}_{\{\bm{\pi}_E(\bm{a}|s)=0\}}\mathbb{I}_{\{\hat{\bm{\pi}}(s,\bm{a}) > 0\}} \\
    & + \gamma\sum\nolimits_{s^\prime} V(s^\prime)(P(s^\prime|s,\bm{a})-\hat{P}(s^\prime|s,\bm{a})).
\end{align*}

The results follow by taking the absolute value and applying triangle inequality.
\end{proof}

Finally, we present the proof of Proposition \ref{proposition2}.

\begin{proof}
For $g\in\mathcal{G}$ and $(s,\bm{a}) \in \tau_E^\mathcal{G}$, we have $\hat{\bm{\pi}}_E^\mathcal{G}(g\bm{a}|gs)=\hat{\bm{\pi}}_E^\mathcal{G}(\bm{a}|s)$ and $\hat{P}^\mathcal{G}(gs^\prime|gs,g\bm{a})=\hat{P}^\mathcal{G}(s^\prime|s,\bm{a})$ since $\hat{\mathcal{M}}^\mathcal{G}$ is the $\mathcal{G}$-invariant MG\textbackslash R. For $(s,\bm{a})\in\tau_E$, we have $\hat{\bm{\pi}}_E^\mathcal{G}(\bm{a}|s)=\hat{\bm{\pi}}_E(\bm{a}|s)$ and $\hat{P}^\mathcal{G}(s^\prime|s,\bm{a})=\hat{P}(s^\prime|s,\bm{a})$, where $\hat{P}^\mathcal{G}$ is the transition model formed from $\tau_E^\mathcal{G}$. The Error Propagation of the MIRL problem $\hat{B}^\mathcal{G}=({\hat{\mathcal{M}}^\mathcal{G}, \hat{\bm{\pi}}_E^\mathcal{G}})$ is:
\begin{align*}
    |r(s,\bm{a})-\hat{r}^\mathcal{G}(s,\bm{a})| \leq & \zeta \mathbb{I}_{\{\bm{\pi}_E(a|s)=0\}}\mathbb{I}_{\{\hat{\bm{\pi}}_E^\mathcal{G}(a|s) > 0\}} + \\ &\gamma\sum_{s^\prime} |V(s^\prime)(P(s^\prime|s,a)-\hat{P}^\mathcal{G}(s^\prime|s,a))|.
\end{align*}
We denote the difference between the upper bounds of the two Error Propagation as:
\begin{align*}
    \delta=&\zeta\mathbb{I}_{\{\bm{\pi}(\bm{a}|s)=0\}}(\mathbb{I}_{\{\hat{\bm{\pi}}_E(\bm{a}|s) > 0\}}-\mathbb{I}_{\{\hat{\bm{\pi}}_E^\mathcal{G}(\bm{a}|s) > 0\}})+\\
           &\gamma\sum_{s^\prime} (|\Delta \hat{P} (s^\prime|s,\bm{a})|-|\Delta \hat{P}^\mathcal{G} (s^\prime|s,\bm{a})|)|V(s^\prime)|,
\end{align*}
where $\delta = |r(s,\bm{a})-\hat{r}(s, \bm{a})|-|r(s,\bm{a})-\hat{r}^\mathcal{G}(s,\bm{a})|$, $\Delta \hat{P} =P-\hat{P}$ and $\Delta \hat{P}^\mathcal{G}=P-\hat{P}^\mathcal{G}$. We will discuss the two terms of $\delta$ separately. 

The first term can only be positive if $\bm{\pi}_E(\bm{a}|s) = 0$, so we begin with this condition. When $(s, \bm{a}) \in \tau_E$, since $\bm{\pi}_E(\bm{a}|s) = 0$, there will be no samples of the pair $(s, \bm{a})$ in $\tau_E$, which implies $\hat{\bm{\pi}}_E(\bm{a}|s) = 0$, and similarly $\hat{\pi}_E^\mathcal{G}(\bm{a}|s) = 0$. In this case, the first term is 0. When $(s, \bm{a}) \in \tau_E^\mathcal{G}$ and $(s, \bm{a}) \notin \tau_E$, we have $\hat{\bm{\pi}}_E(\bm{a}|s) > 0$. Based on the $\mathcal{G}$-invariant MG\textbackslash R, there exists a $g \in \mathcal{G}$ such that $(g^{-1}s, g^{-1}\bm{a}) \in \tau_E$. Therefore, $\bm{\pi}_E(g^{-1}s|g^{-1}\bm{a}) = \bm{\pi}_E(s|\bm{a}) = 0$, which leads to $\hat{\bm{\pi}}_E^\mathcal{G}(g^{-1}\bm{a}|g^{-1}s) = \bm{\pi}_E(g^{-1}s|g^{-1}\bm{a}) = 0$. In this case, the first term is positive. Overall, the first term is non-negative.

According to \cite{auer2008near, osband2017posterior}, we can obtain the bounds:
\begin{align*}
    &\forall (s, \bm{a}) \in \tau, |P(s^\prime|s, \bm{a}) - \hat{P}(s^\prime|s, \bm{a})| \leq \frac{C}{\sqrt{D(s, \bm{a})}} < 1, \\
    &\forall (s, \bm{a}) \notin \tau, |P(s^\prime|s, \bm{a}) - \hat{P}(s^\prime|s, \bm{a})| \leq 1,
\end{align*}
where $C$ is a constant and $D(s, \bm{a})$ is the number of of $(s, \bm{a})$ in $\tau$. 
For the second term, when $(s, \bm{a}) \in \tau_E$, the second term is zero. When $(s, \bm{a}) \in \tau_E^\mathcal{G}$ and $(s, \bm{a}) \notin \tau_E$, the upper bound difference of the second term is $1 - \frac{C}{\sqrt{D^\mathcal{G}(s,\bm{a})}} > 0$. Therefore, the second term is non-negative. In summary, $\delta \geq 0$.
\end{proof}

\addtolength{\textheight}{0cm}   

\section*{Acknowledgment}
This work is supported by the fundamental research funds for the central universities.



\bibliographystyle{IEEEtran}
\bibliography{IEEEexample}

\end{document}